\documentclass[letterpaper, 10 pt, conference]{ieeeconf}  % Comment this line out
\IEEEoverridecommandlockouts                              % This command is only
\def\ARXIV{0}

\usepackage[style=ieee,backend=bibtex]{biblatex}
\usepackage{amsmath}
\usepackage{amssymb}
\usepackage{dsfont}
\usepackage{mathabx} % for \widecheck
\usepackage[scr=esstix]{mathalpha}
\usepackage{upgreek}
\usepackage{bm}
\usepackage{hyperref}
\hypersetup{colorlinks,
	linkcolor=blue,
	citecolor=blue,
	urlcolor=blue,
	linktocpage,
	plainpages=false}
\usepackage{cancel}
\usepackage{graphicx}
\usepackage{mathtools}
\usepackage{enumerate}
\usepackage{float}
\usepackage{tikz}
\usetikzlibrary{shapes,matrix,arrows,calc,positioning,fit,bayesnet,angles,patterns}

\tikzset{%
  every neuron/.style={
    circle,
    draw,
    minimum size=0.5cm
  },
  neuron missing/.style={
    draw=none, 
    scale=2,
    text height=0.333cm,
    execute at begin node=\color{black}$\vdots$
  },
}
\usepackage{CoOL-bib/customCommands}

\newtheorem{theorem}{Theorem}

\newtheorem{lemma}{Lemma}

\newtheorem{assumption}{Assumption}

\begin{document}

\title{\LARGE \bf
  Function Gradient 
  Approximation with Random Shallow ReLU Networks with Control Applications
}

\author{
Andrew Lamperski and 
  Siddharth Salapaka 
  % <-this % stops a space
  \thanks{This work was supported in part by NSF ECCS 2412435}% <-this % stops a space
  \thanks{A. Lamperski is with the department of Electrical and Computer Engineering, University  of Minnesota, Minneapolis, MN 55455, USA {\tt\small alampers@umn.edu}}
  \thanks{S. Salapaka is  with the department of Electrical and Computer Engineering, University of Illinois Urbana-Champaign, Urbana, IL 61801, USA {\tt\small svs9@illinois.edu}}
 % \thanks{T. Lekang is with Honeywell Aerospace, Plymouth, MN 55441, USA
  %      {\tt\small tylerlekang@gmail.com}}
%      \thanks{T. Lekang and A. Lamperski are with the department of Electrical and
%        Computer Engineering, University of Minnesota, Minneapolis,
%        MN 55455, USA 
%        {\tt\small lekang@umn.edu, alampers@umn.edu}}
}

\maketitle
\thispagestyle{empty}
\pagestyle{empty}

\begin{abstract}
  Neural networks are widely used to approximate unknown functions in control. 
  A common neural network architecture uses a single hidden layer (i.e. a shallow network), in which the input parameters are fixed in advance and only the output parameters are trained.
  The typical formal analysis asserts that if output parameters exist to approximate the unknown function with sufficient accuracy, then desired control performance can be achieved.
  A long-standing theoretical gap was that no conditions existed to guarantee that, for the fixed input parameters,  required accuracy could be obtained by training the output parameters.
  Our recent work has partially closed  this gap by demonstrating that if input parameters are chosen randomly, then for any sufficiently smooth function, with high-probability there are output parameters resulting in $O((1/m)^{1/2})$ 
   approximation errors, where $m$ is the number of neurons. 
  However, some applications, notably continuous-time value function approximation, require that the network approximates the both the
  unknown function and its gradient with sufficient accuracy. 
  In this paper, we show that randomly generated input parameters and trained output parameters result in gradient errors of $O\left((\log(m)/m)^{1/2}\right)$, and additionally, improve the constants from our prior work.
  We show how to apply the result to policy evaluation problems. 
\end{abstract}

\section{Introduction}

Neural networks have wide applications in control systems. In adaptive control they are commonly used to model unknown nonlinearities \cite{lavretsky2013robust}. In reinforcement learning and dynamic programming, they are used to approximate value functions and to parameterize control strategies \cite{vrabie2013optimal,powell2007approximate,sutton2018reinforcement}.

Our recent work in \cite{lamperski2024approximation} partially closes a long-standing theoretical gap in neural network methods for control. 
Shallow neural networks are used to approximate unknown functions in a variety of control contexts \cite{vrabie2013optimal,kamalapurkar2018reinforcement,greene2023deep,makumi2023approximate,cohen2023safe,kokolakis2023reachability,sung2023robust,lian2024inverse}.
Shallow neural networks can be expressed succinctly as $\Theta\Phi(Wx+b)$, where $(W,b)$ are the input parameters, $\Phi$ is a vector of nonlinear functions, and $\Theta$ is a matrix of output parameters. In a common setup, the input parameters, $(W,b)$, are fixed in advance and it is \emph{assumed} that for the unknown function, $f$, the approximation error over a bounded set, $\cB$, given by $\inf_{\Theta}\sup_{x\in \cB}\|f(x)-\Theta\Phi(Wx+b)\|$, is small.  
See \cite{vrabie2013optimal,kamalapurkar2018reinforcement,greene2023deep,makumi2023approximate,cohen2023safe,kokolakis2023reachability,sung2023robust,lian2024inverse}. As discussed in \cite{soudbakhsh2023data,annaswamy2023adaptive}, the approximation error relies on appropriate choice of $(W,b)$. While $(W,b)$ classical results guarantee that suitable $(W,b)$ exist \cite{pinkus1999approximation}, no provably correct method to find them was known. Our result in \cite{lamperski2024approximation} shows that if $\Phi$ is constructed from ReLU activations and an affine term and $(W,b)$ are chosen randomly, then with high probability, the approximation error decreases like $O\left((1/m)^{1/2}\right)$, where $m$ is the number of neurons. 

Value function approximation, i.e. policy evaluation, is a common sub-task in approximate dynamic programming and reinforcement learning. In well-known algorithms for continuous-time policy evaluation, discussed in \cite{vrabie2013optimal,kamalapurkar2018reinforcement}, it is required that the value function, $V$, and its gradient can be approximated accurately. In the context of shallow networks with fixed input parameters, the requirement is that
$\inf_{\Theta}\sup_{x\in\cB}\max\{|V(x)-\Theta \Phi(Wx+b)|,\|\nabla_x V(x)-\nabla_x(\Theta\Phi(Wx+b))\|\}$ is small. As in function approximation problems, classical results guarantee that suitable $(W,b)$ exist, but there have been no provably correct methods to find them. This paper shows that if $(W,b)$ are generated randomly, then with high probability, the error in function and gradient approximation decreases like $O\left((\log(m)/m)^{1/2}\right)$.

The main contribution in this paper is the new bound on simultaneous function and gradient approximation using shallow ReLU networks with random input parameters. Additionally, we improve on the function approximation bounds from \cite{lamperski2024approximation}, and show how the results can be used in the analysis of policy evaluation algorithms. As we will see, however, more work is needed to improve the constant factors to get practical bounds. 

The paper is organized as follows. Section~\ref{sec:approximation} gives the general results on approximation theory. Section~\ref{sec:pe} sketches an application to policy evaluation, and Section~\ref{sec:conclusion} gives conclusions.

\section{Function and Gradient Approximation}
\label{sec:approximation}
This section gives quantitative performance bounds on the approximation of a function and its  gradient. 

\subsection{Background}

\paragraph*{Notation}
If $x$ is a vector and $p\in [1,\infty]$,  then $\|x\|_p$ the corresponding $p$-norm. If $M$ is a matrix, then $\|M\|_2$ denotes the induced $2$-norm. $\bbR$ is the set of real numbers and $\bbC$ is the set of complex numbers. $j$ denotes the imaginary unit.
%If $f$ is a real- or complex-valued function and $p\in [1,\infty]$, then $\|f\|_{L_p}$ denotes the corresponding $L_p$ norm.
If $f$ is a scalar-valued  function over a domain $\cX$ with measure $\mu$, then $\|f\|_{L_{\infty}(\cX)}=\mathrm{esssup}_{x\in\cX}|f(x)|$ , where $\mathrm{esssup}$ denotes the essential supremum with respect to measure $\mu$.
If $f$ is a vector-valued function over a domain $\cX$ with measure $\mu$, and $p\in [1,\infty]$, then $\|f\|_{p,L_{\infty}(\cX)}=\mathrm{esssup}_{x\in\cX}\|f(x)\|_p$. Random variables are denoted as bold symbols, e.g $\bx$. The expected value of a random variable, $\bx$,  is denoted by $\bbE[\bx]$ and the probability of an event $\bcE$ is denoted by $\bbP(\bcE)$. The indicator function for an event $\bcE$ is denoted $\indic(\bcE)$, and $\indic(\bcE)$ takes the value of $1$ when $\bcE$ holds and $0$ otherwise.

If $f:\bbR^n \to \bbC$, it is related to its Fourier transform $\hat f:\bbR^n\to \bbC$ by
\begin{subequations}
\begin{align}
  \label{eq:ft}
  \hat f(\omega)&=\int_{\bbR^n}e^{-j2\pi \omega^\top x} f(x)dx\\
  \label{eq:ift}
  f(x)&=\int_{\bbR^n}e^{j2\pi \omega^\top x}\hat f(\omega)d\omega.
\end{align}
\end{subequations}

The main technical assumption on $f$ is the following smoothness condition.

\begin{assumption}
  \label{as:smoothness}
  There is an integer $k\ge n+3$ and a number $\rho >0$ such that 
  $$
  \sup_{\omega\in\bbR^n} |\hat f(\omega)| (1+\|\omega\|_2^k) \le \rho.
  $$
\end{assumption}

Let $\bbS^{n-1}=\{x\in\bbR^n|\|x\|_2=1\}$ denote the $n-1$-dimensional unit sphere, and let $\mu_{n-1}$ denote the Lebesgue measure over $\bbS^{n-1}$. (When $n=1$, $\bbS^{0}=\{-1,1\}$ and $\mu_0$ is the counting measure.)
We denote the area of the area of $\bbS^{n-1}$ by:
\begin{equation*}
  \label{eq:sphereArea}
  A_{n-1} = \frac{2\pi^{n/2}}{\Gamma(n/2)}
\end{equation*}
where $\Gamma$  is the gamma function. The area is maximized at $n=7$, and decreases geometrically with $n$.

  % \begin{lemma}
  %  \label{lem:lipschitz}
  %   {\it
  %   If $f$ satisfies Assumption~\ref{as:smoothness}, then $f$ is $2(2\pi)\rho A_{n-1}$-Lipschitz and $\nabla f$ is $2(2\pi)^2 \rho A_{n-1}$-Lipschitz. 
  % }
  % \end{lemma}

  % \begin{proof}
  %   The inverse formula in \eqref{eq:ift} shows that 
  %   the gradient and Hessian $f$ are given respectively by
  %   \begin{align*}
  %     \nabla f(x)&=\int_{\bbR^n}e^{j2\pi \omega^\top x} (j2\pi \omega) \hat f(\omega)d\omega \\
  %     \nabla^2 f(x)&=\int_{\bbR^n}e^{j2\pi \omega^\top x}(j2\pi \omega)(j2\pi \omega)^\top \hat f(\omega)d\omega.
  %   \end{align*}

  %   It was shown in \cite{lamperski2024approximation} that if Assumption \ref{as:smoothness} holds, then 
  %   $$
  %   \int_{\bbR^n}|\hat f(\omega)|\|2\pi \omega\|^i \le 2(2\pi)^i \rho A_{n-1}
  %   $$
  %   for $i=0,1,2$.

  %   The result now follows, since  $\|\nabla f(x)\|\le 2(2\pi) \rho A_{n-1}$ and $\|\nabla^2 f(x)\| \le 2 (2\pi)^2 \rho A_{n-1}$ for all $x\in\bbR^n$.
  % \end{proof}

Let $\sigma$ denote the ReLU activation function and let $\sigma'$ denote the unit step, which is the derivative of $\sigma(t)$ for all $t\ne 0$:
\begin{align*}
  \sigma(t)&=\max\{t,0\}\\
  \sigma'(t)&=\indic(t\ge 0)=\begin{cases}
    0 & t < 0 \\
    1 & t\ge 0.
  \end{cases}
\end{align*}

\subsection{A Function and Gradient Approximation Result}

To present our result, we introduce some notation.

For the rest of this section, $R$ denotes a fixed (but arbitrary) positive number.
Let $\cB=\{x\in\bbR^n|\|x\|_2\le R\}$.
Let $P:\bbS^{n-1}\times [-R,R]\to \bbR$ be a probability density function.
\begin{assumption}
  \label{as:positive}
The density, $P$, has a positive lower bound: $P_{\min}:=\inf_{(\alpha,t)\in\bbS^{n-1}\times[-R,R]}P(\alpha,t)>0$.
\end{assumption}

In the case that $P$ is the uniform distribution, we have that $P(\alpha,t)=P_{\min}=\frac{1}{2R A_{n-1}}$.

The result below is the main approximation result in the paper.
It is proved in Subsection~\ref{ss:pf} after presenting some technical lemmas.

\begin{theorem}
  \label{thm:main}
  {\it
  Let Assumptions~\ref{as:smoothness} and \ref{as:positive} hold. Let 
  $(\balpha_1,\bt_1),\ldots,(\balpha_m,\bt_m)$ be a collection of independent, identically distributed samples from $P$.
  There is a vector $a\in\bbR^n$ and a number $b\in \bbR$ such that for all $m\ge n+1$ and all $\delta \in (0,1)$ with probability at least $1-\delta$
  there exist
  coefficients $\bc_1,\ldots,\bc_m$ such that for all $\delta \in (0,1)$, the neural network approximation defined by
  \begin{align*}
    \boldf_N(x)=a^\top x +b + \sum_{i=1}^m\bc_i \sigma(\balpha_i^\top x-\bt_i)
  \end{align*}
  satisfies both of the following inequalities simultaneously
  \begin{subequations}
    \begin{align}
      \nonumber
    \|\boldf_N-f\|_{L_{\infty}(\cB)}&\le \frac{1}{\sqrt{m}}\left(1+\kappa_1 \sqrt{\log(2/\delta)} \right)\\
    \label{eq:fApprox}
    &+\kappa_2 \sqrt{\log(2(1+2RL\sqrt{m}))} \\
    \nonumber
    \|\nabla \boldf_N-\nabla f\|_{2,L_{\infty}(\cB)}&\le\\
    \label{eq:gradApprox}
                                    &
                                      \hspace{-40pt}\frac{1}{\sqrt{m}}\left(\zeta_0\sqrt{\log(m+1)}+\zeta_1\sqrt{\log(2n/\delta)} \right)
  \end{align}
  \end{subequations}
  with probability at least $1-\delta$. Here, the numbers $\kappa_0,\kappa_1,\zeta_0,\zeta_1$, and $L$ are defined by
  \begin{align*}
      \kappa_1&=4\beta\\
      \kappa_2&=\beta\sqrt{2n}\\
      L&=\frac{8\pi^2 \rho}{P_{\min}}+8\pi A_{n-1} \rho\\
      \beta&= \frac{16\pi^2 \rho R}{P_{\min}}+(4+8\pi R)A_{n-1}\rho\\
    \zeta_0 &=\frac{64\pi^2 (n+1)\rho}{P_{\min}} \\
    \zeta_1 &= \frac{8\sqrt{2n} \pi^2 \rho}{P_{\min}}.
  \end{align*}
  Furthermore, $a$, $b$, and $\bc_i$ satisfy the following bounds\footnote{The published version of Theorem~1 of \cite{lamperski2024approximation} has a mistake in the bound on $|b|$, but the correct bound is given in Lemma~1 of that paper.}:
      \begin{align*}
      \|a\|_2 &\le 4\pi A_{n-1}\rho \\
      |b| & \le (2+4\pi R) A_{n-1}\rho \\
      |\bc_i| & \le \frac{8\pi^2 \rho}{m P_{\min}}
    \end{align*}
}
\end{theorem}

\subsection{Technical Lemmas}
\begin{lemma}
  \label{lem:integral}
  {\it
    Let $f:\bbR^n\to \bbR$ satisfy Assumption~\ref{as:smoothness}. For any $R>0$, 
    there is a function $g:\bbS^{n-1}\times [-R,R]\to \bbR$, a vector $a\in\bbR^n$, and a scalar $b\in\bbR$ such that for all $\|x\|_2\le R$
  \begin{subequations}
  \begin{align}
    \label{eq:represent}
    \nonumber
    f(x)&=\int_{\bbS^{n-1}}\int_{-R}^R g(\alpha,t)\sigma(\alpha^\top x-t)  dt \mu_{n-1}(d\alpha) \\
        &+a^\top x +b \\
    \label{eq:integralRepresent}
    \nabla f(x)&= \int_{\bbS^{n-1}}\int_{-R}^R g(\alpha,t)\sigma'(\alpha^\top x-t)\alpha  dt \mu_{n-1}(d\alpha) +a.
  \end{align}
\end{subequations}
Furthermore, $g$, $a$, and $b$ satisfy:
\begin{align*}
  \|g\|_{L_{\infty}}&\le 8\pi^2\rho \\
  \|a\|_2&\le 4\pi \rho A_{n-1}\\
  |b|&\le 2\rho A_{n-1}\left( 1+2\pi R\right). 
\end{align*}
  }
\end{lemma}

\begin{proof}
  The integral representation  \eqref{eq:represent} and the bounds were proved in Lemma 1 of \cite{lamperski2024approximation}.

  The integral representation in~\ref{eq:integralRepresent} now follows by differentiating \eqref{eq:represent}.  
\end{proof}

%$P_{\max}=\sup_{(\alpha,t)\in \bbS^{n-1}\times[-R,R]}P(\alpha,t)$, and $P_{[-R,R],\max}=\sup_{t\in [-R,R]}P_{[-R,R]}(t)$.

As in \cite{lamperski2024approximation}, we use the integral representation from Lemma~\ref{lem:integral} to form the importance sampling estimate:
\begin{align}
  \label{eq:importance}
  \boldf_I(x)=a^\top x + b + \frac{1}{m}\sum_{i=1}^m\frac{g(\balpha_i,\bt_i)}{P(\balpha_i,\bt_i)}\sigma(\balpha_i^\top x-\bt_i),
\end{align}
where $(\balpha_1,\bt_1),\ldots,(\balpha_m,\bt_m)$ are independent, identically distributed samples from $P$ and $m$ is a positive integer.

The following result is a variation on Lemma~2 of \cite{lamperski2024approximation} with improved constants. While this new result has a factor of $O(\sqrt{\log(m)})$, in numerical tests, the result below gives better bounds. 

\begin{lemma}
  \label{lem:importance}
  {\it
    Assume that $P_{\min}>0$ and let $f$ satisfy Assumption~\ref{as:smoothness}. 
    For all $\delta \in (0,1)$ and all $m\ge 1$, the following bound holds with probability at least $1-\delta$:
    \begin{multline*}
      \|\boldf_I-f\|_{L_{\infty}(\cB)}\le
      \\ \frac{1}{\sqrt{m}}\left(1+\kappa_1\sqrt{\log(\delta^{-1})}+\kappa_2
                                        \sqrt{\log(2(1+2RL\sqrt{m}))}
                                        \right)
    \end{multline*}
    where
    \begin{align*}
      \kappa_1&=4\beta\\
      \kappa_2&=\beta\sqrt{2n}\\
      L&=\frac{8\pi^2 \rho}{P_{\min}}+8\pi A_{n-1} \rho\\
      \beta&= \frac{16\pi^2 \rho R}{P_{\min}}+(4+8\pi R)A_{n-1}\rho.
    \end{align*}
  }
\end{lemma}

\begin{proof}
  Define the random functions $\btheta$ and $\bxi_i$ by
  \begin{align*}
    \btheta(x) &=\frac{1}{m}\sum_{i=1}^m\bxi_i(x) \\
    \bxi_i(x)&=\frac{g(\balpha_i)}{P(\balpha_i,\bt_i)}\sigma(\balpha_i^\top x-\bt_i)
               +a^\top x +b
               - f(x).
  \end{align*}
  Lemma~\ref{lem:integral} implies that $\bxi_i(x)$ have zero mean for all $\|x\|\le R$. Note that $\btheta(x)=\boldf_I(x)-f(x)$.

  We will bound $\sup_{x\in\cB}|\btheta(x)|$ with high probability using the functional Hoeffding inequality, which is Theorem 3.26 of \cite{wainwright2019high}. This approach gives smaller constants than the method from \cite{lamperski2024approximation}.

Equation~(6a) of \cite{lamperski2024approximation} shows that $\|\hat f\|_{L_1}\le 2\rho A_{n-1}$, which implies that $\|f\|_{L_{\infty}}\le 2 \rho A_{n-1}$. So, using the bounds from Lemma~\ref{lem:integral} above gives for all $\|x\|_2\le R$:
\begin{align*}
  |\bxi_i(x)|\le \frac{\|g\|_{L_{\infty}}2R}{P_{\min}}+\|a\|_2 R +|b|+\|f\|_{L_{\infty}}\le \beta,
\end{align*}
where $\beta = \frac{16\pi^2 \rho R}{P_{\min}}+(4+8\pi R)A_{n-1}\rho$.

  Let
  $$
  \bz=\sup_{x\in\cB}|\btheta(x)|=\sup_{(x,s)\in\cB\times \{-1,1\}}\frac{1}{m}\sum_{i=1}^ms\bxi_i(x).
  $$
  Then since $|s\bxi_i(x)|\le \beta$, the functional Hoeffding inequality implies that for all $t\ge 0$:
  \begin{equation*}
    \bbP(\bz\ge \bbE[\bz]+t)\le \exp\left(-\frac{m t^2}{16\beta^2} \right).
  \end{equation*}
  For $\delta\in (0,1)$, we set $\exp\left(-\frac{m t^2}{16\beta^2} \right) =\delta$ and re-arrange to give
  \begin{equation}
    \label{eq:funHoeffdingInvert}
    \bbP\left(\bz\ge \bbE[\bz]+\frac{4\beta\sqrt{\log(\delta^{-1})}}{\sqrt{m}}\right)\le \delta.
  \end{equation}

  The result will now follow after bounding $\bbE[\bz]$.
  
  %To use (\ref{eq:funHoeffdingInvert}), we will bound $\bbE[\bz]$. 

A zero-mean scalar random variable, $\bv$ is called $\sigma$-sub-Gaussian if $\bbE[e^{\lambda \bv}]\le e^{\frac{\lambda^2 \sigma^2}{2}}$. For all $x\in\cB$, Hoeffding's lemma implies that $\bxi_i(x)$ is $\beta$-sub-Gaussian, and so a standard calculation (e.g. Exercise 2.3 of \cite{wainwright2019high}) shows that $\btheta(x)$ is $\frac{\beta}{\sqrt{m}}$-sub-Guassian.

For any $\epsilon>0$, an $\epsilon$-covering of $\cB$ is a collection of points $\cC=\{x_1,\ldots,x_Q\}\subset\cB$ such that for all $x\in\cB$, there is an $x_i\in\cC$ such that $\|x-x_i\|_2\le \epsilon$. For all $\epsilon >0$, the unit ball of $\bbR^n$ has an $\epsilon$-covering with cardinality at most $\left(1+\frac{2}{\epsilon}\right)^{n}$. (See Example~5.8 of \cite{wainwright2019high}.) By rescaling, $\cB$ has an $\epsilon$-covering, $\hat\cC_{\epsilon}$, of cardinality at most $\left(1+\frac{2R}{\epsilon}\right)^n$. Now, $\cC_{\epsilon}:=\hat\cC_{\epsilon}\times \{-1,1\}$ is an $\epsilon$-covering of $\cB\times\{-1,1\}$ of cardinality at most $2\left(1+\frac{2R}{\epsilon}\right)^n$.

For any $(x,s)\in\cB\times \{-1,1\}$, there is an element $(x_i,s)\in\cC_{\epsilon}$ such that $\|x-x_i\|_2\le \epsilon$. It was shown in Lemma~2 in \cite{lamperski2024approximation} that $\bxi_i$ are $L$-Lipschitz, with $L=\frac{8\pi^2 \rho}{P_{\min}}+8\pi A_{n-1} \rho$. Thus, $\btheta$ is also $L$-Lipschitz. It follows that
$$
s\btheta(x)\le s\btheta(x_i)+|\btheta(x)-\btheta(x_i)|\le s\btheta(x_i)+L\epsilon. 
$$
Maximizing the left over $\cB\times\{-1,1\}$ and the right over $\cC_{\epsilon}$ shows that
$$
\bz\le L\epsilon+\max_{(x_i,s)\in\cC_{\epsilon}}s\btheta(x_i).
$$
Then, since $s\btheta(x_i)$ are all $\frac{\bbeta}{\sqrt{m}}$-sub-Gaussian, Exercise 2.12 of \cite{wainwright2019high} implies that
\begin{align*}
  \bbE[\bz]&\le L\epsilon+\sqrt{\frac{2\beta^2 \log\left(|\cC_{\epsilon}| \right)}{m}}.
  %\\
  %&\le L\epsilon+\sqrt{\frac{2n\beta^2\log(2)}{m}}+\sqrt{\frac{4nR\beta^2}{\epsilon m}}
\end{align*}
Choosing $\epsilon=\frac{1}{L\sqrt{m}}$ results in
$$
\bbE[\bz]\le\frac{1}{\sqrt{m}}\left(
  1+\beta\sqrt{2n\log\left(2\left(1+2RL\sqrt{m}\right)\right)}
\right).
$$
Plugging this bound into (\ref{eq:funHoeffdingInvert}) completes the proof.
\end{proof}

Now we turn to the problem of gradient approximation. 
For any fixed collection of samples, we have that for almost all $x$:
\begin{align*}
  \nabla\boldf_I(x)=a + \frac{1}{m}\sum_{i=1}^m\frac{g(\balpha_i,\bt_i)}{P(\balpha_i,\bt_i)}
  \indic(\balpha_i^\top x\ge\bt_i)
  %\sigma'(\balpha_i^\top x-\bt_i)
  \balpha_i. 
\end{align*}
Alternatively, for any fixed $x$, the equality above holds almost surely.

%Lemma~\ref{lem:importance} above gives quantitative bounds on how well $\boldf_I$ approximates $f$, similar to Lemma~2 of \cite{lamperski2024approximation}.
The result below gives bounds on $\nabla\boldf_I-\nabla f$. The proof requires fundamentally different techniques from Lemma~\ref{lem:importance} above or the corresponding result of \cite{lamperski2024approximation}. Indeed, the proof of Lemma~\ref{lem:importance} uses Lipschitz-continuity of all the terms in $\boldf_I$, but $\nabla \boldf_I$ is discontinuous, due to discontinuity of $\sigma'$. 

\begin{lemma}
  \label{lem:gradApprox}
  {\it
    Assume that
    $P_{\min}>0$ and let $f$ satisfy Assumption~\ref{as:smoothness}. Then for all $m\ge n+1$, and all $\delta \in (0,1)$ the following bounds hold with probability at least $1-\delta$:
    \begin{align*}
      \left\| \nabla \boldf_I-\nabla f\right\|_{\infty,L_{\infty}(\cB)}\le \\
      &
        \hspace{-70pt}
 \frac{1}{\sqrt{m}}\frac{8\pi^2\rho}{P_{\min}}\left(8\sqrt{(n+1)\log(m+1)}+\sqrt{2\log(n/\delta)}\right)\\
      \left\| \nabla \boldf_I-\nabla f\right\|_{2,L_{\infty}(\cB)}\le \\
      &
        \hspace{-70pt}
        \frac{\sqrt{n}}{\sqrt{m}}
        \frac{8\pi^2\rho}{P_{\min}}
        \left(8\sqrt{(n+1)\log(m+1)}+\sqrt{2\log(n/\delta)}\right).
    \end{align*}
  }
\end{lemma}

\begin{proof}
  The proof is structured as follows. First, for a fixed vector, $v$, with $\|v\|_2=1$, we will prove a concentration result on $\left(\nabla \boldf_I(x)-\nabla f(x)\right)^\top v$. Then, we will specialize the result to $v=e_i$ where $e_i$ are the standard basis vectors, and prove the desired results with a union bound.
  
  Fix a unit vector $v\in\bbR^n$, set $\btheta_i=(\balpha_i,\bt_i)$, and define
  \begin{align*}
    h(\btheta_i,x,v)&=\frac{g(\balpha_i,\bt_i)}{P(\balpha_i,\bt_i)}(\balpha_i^\top v)\indic(\balpha_i^\top x\ge\bt_i) \\
    \cH_v &= \{h(\cdot,x,v)|\|x\|_2\le R\}.
  \end{align*}

  By construction, for all $x$, we have almost surely
  \begin{multline*}
    (\nabla \boldf_I(x)-\nabla f(x))^\top v=\\
    \left(\frac{1}{m}\sum_{i=1}^m h(\btheta_i,x,v)\right)-(\nabla f(x)-a)^\top v.
  \end{multline*}
  Furthermore, both sides have zero mean.
  %and for almost all $x$:
  %\begin{align*}
  %  \MoveEqLeft
  %\|\nabla \boldf_I(x)-\bold f(x)\|=\sup_{\|v\|=1} (\nabla \boldf_I(x)-\nabla f(x))^\top v \\
  %&= \sup_{\|v\|=1}\left( \left(\frac{1}{m}\sum_{i=1}^m h(\btheta_i,x,v)\right)-(\nabla f(x)-a)^\top v\right).
  %\end{align*}

    For compact notation, set $c=\frac{8\pi^2\rho}{P_{\min}}$. Then, since $v$ is a unit vector, we have that
  $$
  \sup_{(\alpha,t,x)\in\bbS_{n-1}\times [-R,R]\times \bbR^n}|h((\alpha,t),x,v)|\le c.
  $$

  We will bound
  \begin{multline*}
    \xi(\btheta_{1:m},v):=\\
    \sup_{\|x\|_2\le R}\left| \left(\frac{1}{m}\sum_{i=1}^m h(\btheta_i,x,v)\right)-(\nabla f(x)-a)^\top v\right|
  \end{multline*}
  using methods from classical learning theory, which are discussed in Chapter 4 of \cite{wainwright2019high}. 

  Theorem 4.10 of \cite{wainwright2019high} shows that for $\delta \in (0,1)$, with probability at least $1-\delta$:
  \begin{equation}
    \label{eq:rademacherToEmpMean}
    %\sup_{\|x\|\le R}\left| \left(\frac{1}{m}\sum_{i=1}^m h(\btheta_i,x,v)\right)-(\nabla f(x)-a)^\top v\right|\\
    \xi(\btheta_{1:m},v)
    \le 2\cR_m(\cH_v)+\frac{c\sqrt{2\log(\delta^{-1})}}{\sqrt{m}},
  \end{equation}
  where $\cR_m(\cH_v)$ is the \emph{Rademacher complexity} of the function class $\cH_v$. To define the Rademacher complexity, let $\bepsilon_i$ be independent, identically distributed random variables, independent of $\btheta_1,\ldots,\btheta_m$, such that $\bepsilon_i=1$ with probability $1/2$ and $\bepsilon_i=-1$ with probability $1/2$. Then
  $$
  \cR_m(\cH_v)=\bbE\left[\sup_{\|x\|_2\le R}\left|\frac{1}{m}\sum_{i=1}^m \bepsilon_i h(\btheta_i,x,v)\right|\right].
  $$

  Bounds on the Rademacher complexity can be found by fixing the data set $\btheta_1,\ldots,\btheta_m$, and then bounding the number of distinct vectors $[h(\btheta_1,x,v),\ldots,h(\btheta_m,x,v)]$ that can be obtained as $x$ ranges over its  domain:
  \begin{align}
    \nonumber
    \MoveEqLeft[0]
    \left|\left\{[h(\btheta_1,x,v),\ldots,h(\btheta_m,x,v)] \middle| \|x\|_2\le R\right\} \right| \\
    \nonumber
    &\le \left|\left\{[\indic(\balpha_1^\top x \ge \bt_1),\ldots, \indic(\balpha_m^\top x\ge \bt_m)] \middle| \|x\|_2\le R \right\} \right| \\
    %\label{eq:indicBound}
    \label{eq:linBound}
    &\le 
      \left|\left\{[\indic(\balpha_1^\top x \ge y\bt_1),\ldots, \indic(\balpha_m^\top x\ge y\bt_m)] \middle| \begin{bmatrix}x\\ y\end{bmatrix}\in\bbR^{n+1} \right\} \right|.
    %\label{eq:VCbound}
    %&\le (m+1)^{n+1}.
  \end{align}

  %Before proving (\ref{eq:VCbound}), we note that 

  We will now bound the cardinality of the set on the right of (\ref{eq:linBound}) using a VC-dimension argument. For compact notation, set $z=[x^\top, y]^\top$. 
  Let $\cL$ be the class of functions $\ell(\cdot,z):\bbS^{n-1}\times [-R,R]\to \bbR$ defined by
  $$
  \ell(\theta,z)=yt-\alpha^\top x.
  $$
  Note that the functions in $\cL$ are in $1-1$ correspondence with $z\in\bbR^{n+1}$. 
  
  For a collection of points $\theta_{1:q}=(\theta_1,\ldots,\theta_q)$, define 
  \begin{multline*}
  \cL(\theta_{1:q}):=\\
  \left\{[\indic(\ell(\theta_1,z)\le 0),\ldots, \indic(\ell(\theta_q,z)\le 0)] \middle| z\in\bbR^{n+1} \right\}. 
\end{multline*}
Note that the right of (\ref{eq:linBound}) is precisely $|\cL(\btheta_{1:m})|$.
%Then the bound in (\ref{eq:VCbound}) is equivalent to $|\cL(\btheta_{1:m})|\le (m+1)^{n+1}$.

The class of functions, $\cL$, is said to \emph{shatter} a collection of points $\theta_{1:q}$ if $|\cL(\theta_{1:q})|=2^q$. The \emph{VC-dimension} (Vapnik-Chervonenkis), denoted by $\mathrm{VC}(\cL)$,  is the largest number $q$, such that there exists a collection of points of size $q$, $\theta_{1:q}$, such that  $\cL$ shatters $\theta_{1:q}$. 
The Vapnik, Chervonenkis, Sauer, Shelah Theorem states that for any $q\ge \mathrm{VC}(\cL)$:
$$
|\cL(\theta_{1:q})|\le (q+1)^{\mathrm{VC}(\cL)}.
$$

%The set $\cL(\theta_{1:q})$ is defined similarly for a collection of points, $\theta_{1:q}=(\theta_1,\ldots,\theta_q)$. The class $\cL$ is said to \emph{shatter} a  collection of points $\theta_{1:q}=(\theta_1,\ldots,\theta_q)$ if $|\cL(\theta_{1:q})|=2^q$. 

Now, since $\cL$ is a vector space of dimension $n+1$, Proposition 4.20 of \cite{wainwright2019high} shows that $\mathrm{VC}(\cL)\le n+1$. It follows, in particular, that $|\cL(\btheta_{1:m})|\le (m+1)^{n+1}$ when $m\ge n+1$.

Using that the set on the left of (\ref{eq:linBound}) must also be bounded by $(m+1)^{n+1}$,  Lemma 4.14 of \cite{wainwright2019high} shows that
$$
\cR_m(\cH_v)\le 4c \sqrt{\frac{(n+1)\log(m+1)}{m}}.
$$
Plugging the bound on the Rademacher complexity in (\ref{eq:rademacherToEmpMean}) shows that
\begin{align}
  %\sup_{\|x\|\le R}\left| \left(\frac{1}{m}\sum_{i=1}^m h(\btheta_i,x,v)\right)-(\nabla f(x)-a)^\top v\right|\\
  \nonumber
  \xi(\btheta_{1:m},v)
  &\le
    \frac{c}{\sqrt{m}}\left(8\sqrt{(n+1)\log(m+1)}+\sqrt{2\log(\delta^{-1})}\right)\\
  \label{eq:empMeanBound}
  &=:\tau(m,\delta).
\end{align}
with probability at least $1-\delta$. The bounding function $\tau$ was defined for compact notation in the argument below.

Now, we complete the proof with a union bounding argument. First, we bound $\|\nabla \boldf_I-\nabla f\|_{\infty,L_{\infty}}$, which can be expressed as:
\begin{align*}
  \MoveEqLeft
  \|\nabla \boldf_I - \nabla f\|_{\infty,L_{\infty}}(\cB)\\
  &=\mathrm{esssup}_{\|x\|\le R} \|\nabla \boldf_I(x)-\nabla f(x)\|_{\infty} \\
                                                    &=\mathrm{esssup}_{\|x\|\le R}\max_{k=1,\ldots,n} \left|(\nabla \boldf_I(x)-\nabla f(x))_k \right| \\
                                                    &=\max_{k=1,\ldots,n}\xi(\btheta_{1:m},e_k).
\end{align*}
Here the  essential supremum is taken with respect to Lebesgue measure. (The calculations with $\nabla \boldf_I$ require using an essential supremum, since $\nabla \boldf_I(x)$ is defined for almost all, but not all $x$. However, a regular supremum was used in the definition of $\xi(\btheta_{1:m},v)$, since the corresponding term being maximized is well-defined for  all $x$)

The bound in (\ref{eq:empMeanBound}) holds for all unit vectors $v$ and arbitrary $\delta \in (0,1)$. In particular, let $e_1,\ldots,e_n$ denote the standard basis vectors of $\bbR^n$. Then for $k=1,\ldots,n$, we have that
$$
\xi(\btheta_{1:m},e_k)> \tau(m,\delta/n)
$$
with probability at most $\delta/n$. It follows that
\begin{align*}
  \MoveEqLeft
  \bbP\left(
  \|\nabla \boldf_I-\nabla f\|_{\infty,L_{\infty}(\cB)} > \tau(m,\delta/n)
  \right) \\
  &=
    \bbP\left(
   \max_{k=1,\ldots,n} \xi(\btheta_{1:m},e_k)> \tau(m,\delta/n)
    \right) \\
  &\le \sum_{k=1}^n \bbP\left(
    \xi(\btheta_{1:m},e_k)> \tau(m,\delta/n)
    \right) \\
  &\le \delta,
\end{align*}
where the first inequality is due to a union bound. 

The bound on
$$
\|\nabla \boldf_I-\nabla f\|_{2,L_{\infty}(\cB)}=\mathrm{esssup}_{x\in\cB}\|\nabla \boldf_I(x)-\nabla f(x)\|_2
$$
is immediate, since $\|w\|_2\le \sqrt{n}\|w\|_{\infty}$ for any $w\in\bbR^n$. 
\end{proof}

\subsection{Proof of Theorem~\ref{thm:main}}
\label{ss:pf}

Let $a$ and $b$ be the vector and scalar guaranteed from Lemma~\ref{lem:integral} and set $\bc_i=\frac{g(\balpha_i,\bt_i)}{mP(\balpha_i,\bt_i)}$. Then $\boldf_N=\boldf_I$, where $\boldf_I$ was defined in (\ref{eq:importance}). 

Fix $\delta \in (0,1)$. 
Theorem 1 of \cite{lamperski2024approximation} shows that (\ref{eq:fApprox}) holds with probability at least
$1-(\delta/2)$. Furthermore, Theorem 1 of \cite{lamperski2024approximation} gives
the bounds on $\|a\|_2$, $|b|$, and $|\bc_i|$.

Lemma~\ref{lem:gradApprox} implies that (\ref{eq:gradApprox}) holds with probability at least $1-(\delta/2)$.

Now, a union bound implies that the event that either (or both) of (\ref{eq:fApprox}) or (\ref{eq:gradApprox}) fails occurs with probability at most $\delta$. Thus, \eqref{eq:fApprox} and \eqref{eq:gradApprox} must hold simultaneously with probability at least $1-\delta$. 
\hfill\QED

\section{Application to Policy Evaluation}
\label{sec:pe}

\subsection{General Theory}
Consider a control-affine dynamical system in continuous time:
$$
\frac{dx_t}{dt}=f(x_t)+g(x_t)u_t
$$
with infinite-horizon cost:
$$
\int_0^{\infty}\left(q(x_t)+\frac{1}{2}u_t^\top R u_t\right)dt.
$$
(Here $R$ is a matrix, and does not correspond to the bound from the previous section.)

If the inputs follow a fixed policy, $u_t = \phi(x_t)$, the \emph{value} function for policy $\phi$ is given by:
$$
V_{\phi}(x_0)=\int_0^{\infty}\left(q(x_t)+\frac{1}{2}\phi(x_t)^\top R \phi(x_t)\right)dt.
$$
and satisfies the PDE:
$$
q(x)+\frac{1}{2}\phi(x)^\top R \phi(x)+\frac{\partial V_{\phi}(x)}{\partial x}(f(x)+g(x)\phi(x))=0.
$$

\emph{Policy evaluation} is the computation of $V_{\phi}$, and is one of the primary calculations in the dynamic programming method known as policy iteration, and related actor critic methods from reinforcement learning.

In most problems in optimal control and reinforcement learning for nonlinear systems, $V_{\phi}$ cannot be computed  exactly. (In learning contexts, $f$ and $g$ are  typically unknown.) As described in the introduction, a common approach is to utilize a neural network approximation of the form $\Theta \Phi(Wx+b)$, where $(W,b)$ are fixed in advance, and then compute $\Theta$ so that $V_{\phi}(x)\approx \Theta\Phi(Wx+b)$.

The books \cite{vrabie2013optimal,kamalapurkar2018reinforcement} describe various well-known methods to perform approximate policy iteration with neural networks. To get guarantees on the approximation, they assume that over a bounded region, $\cB$:
\begin{multline}
  \label{eq:jointError}
  \inf_{\Theta}\sup_{x\in\cB}\max\left\{|V_{\phi}(x)-\Theta \Phi(Wx+b)|,\right.\\
  \left.\|\nabla_x V_{\phi}(x)-\nabla_x(\Theta\Phi(Wx+b))\|_2\right\} \le \epsilon,
\end{multline}
where $\epsilon$ is a prescribed error tolerance.

These assumptions are partially justified using the higher-order Weierstrauss approximation theorem, in the sense that suitable $(W,b)$ are guaranteed to \emph{exist}. But since there has been no provably correct method to find such $(W,b)$, the justification is incomplete.

Now we will sketch how Theorem~\ref{thm:main} can be used to construct randomized $(\bW,\bb)$ and an appropriate vector of functions $\Phi$ that satisfy the approximation requirements with high probability. Let $m\ge (n+1)$ be a number of neurons. Assume that $V_{\phi}$ satisfies Assumption~\ref{as:smoothness}. 
Use $V_{\phi}$ in place of $f$ in Theorem~\ref{thm:main}. Let $(\balpha_i,\bt_i)$ be independent, identically distributed samples drawn according to $P$. Let $a$, $b$, and $\bc_i$ be the coefficients guaranteed to exist by Theorem~\ref{thm:main}.

Set:
\begin{align*}
  \bW&=\begin{bmatrix}
    0 \\
    I \\
    \balpha_1 \\
    \vdots \\
    \balpha_m
  \end{bmatrix}
  & \bb&=\begin{bmatrix}
    1 \\
    0_{n\times 1} \\
    -\bt_1 \\
    \vdots \\
    -\bt_m
  \end{bmatrix}
  \\
  \Phi(\bW x+\bb)&=\begin{bmatrix}
    1 \\
    x \\
    \sigma(\balpha_1^\top x-\bt_1) \\
    \vdots \\
    \sigma(\balpha_m^\top x-\bt_m)
  \end{bmatrix}
     &
       \bTheta^\top &=\begin{bmatrix}
         b \\
         a \\
         \bc_1 \\
         \vdots \\
         \bc_m
       \end{bmatrix}
\end{align*}

Theorem~\ref{thm:main} implies that there is a constant $c>0$ such that for all $\delta \in (0,1)$, the following bound holds with probability at least $1-\delta$:
\begin{multline}
  \label{eq:PEBound}
  \sup_{x\in\cB}\max\left\{|V_{\phi}(x)-\bTheta \Phi(\bW x+\bb)|,
  \right.\\
  \left.
    \|\nabla_x V_{\phi}(x)-\nabla_x(\bTheta\Phi(\bW x+\bb))\|_2\right\} \\\le c\frac{\sqrt{\log(\delta^{-1})}+\sqrt{\log(m)}}{\sqrt{m}}.
\end{multline}

In particular, when $m$ is sufficiently large, the bound is below $\epsilon$, and so (\ref{eq:jointError}) holds with high probability. Thus, in principle, Theorem~\ref{thm:main} guarantees that randomly generated input parameters suffice to achieve the required approximation properties. In practice, however, the problem is not completely solved, since the constant $c$ is so large that a massive number of neurons is required before the bound on the right of (\ref{eq:PEBound}) drops to a reasonable level. See the numerical example below.

\subsection{Numerical Example}

To test the predictions from Theorem~\ref{thm:main} in a policy evaluation problem, we consider a simple nonlinear problem with a scalar state in which $V_{\phi}$ can be computed analytically:
\begin{align*}
  f(x)&=0\\
  g(x)&=1\\
  \phi(x)&=-\tanh(5x)\\
  q(x)&=\frac{1}{2}x^2\\
  R&=1
\end{align*}
This results in:
$$
V_{\phi}(x)=\frac{1}{5}\log(\cosh(5x)).
$$

Our goal is to determine the error in approximating $V_{\phi}$ and $\nabla_x V_{\phi}$ over the region $[-1,1]$. 

To calculate the error bounds, we need a bound on the smoothness coefficient, $\rho$, which in turn, requires bounding Fourier transforms. Since we only focus on approximating $V_{\phi}$ in the region $[-1,1]$, we examine, instead, $V_{\phi}(x)r(x)$, where the multiplier function, $r(x)$, takes the value $1$ over the region $[-1,1]$, and has a corresponding bounded $\rho$. See Fig.~\ref{fig:Vmod}. In our experiments, we defined $r$ through its Fourier transform, with
$$
\hat r(\omega) = 3\left(\frac{\sin\left(\frac{\pi \omega}{5}\right)}{\frac{\pi \omega}{5}}\right)^5 \frac{\sin\left(3\pi\omega  \right)}{3\pi \omega}.
$$
In this case, $r(x)$ is differentiable up to $5$th order, with $r(x)=1$ over $[-1,1]$ and $r(x)=0$ outside $[-2,2]$.  
Numerical integration gives a bound of $\rho \le 2$, with $k=n+3=4$. 

\begin{figure}
  \centering
  \includegraphics[width=.8\columnwidth]{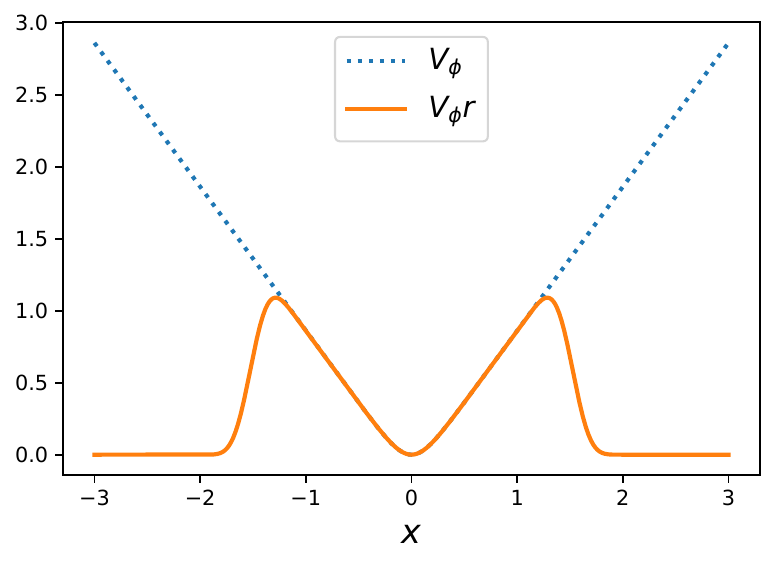}
  \caption{\label{fig:Vmod} {\bf $V_\phi$ and its modified version.}}
\end{figure}

Theorem~\ref{thm:main} gives a bound on the approximation errors achieved by the importance sampling estimate, from (\ref{eq:importance}), which in turn, gives an upper bound on the errors that could be achieved by optimization. In a numerical experiment, we optimized the coefficients via least squares to give the errors shown in Figs.~\ref{fig:Verr} and \ref{fig:GradErr}. As can be seen, the bounds are quite conservative. 

\begin{figure}
  \centering
  \includegraphics[width=.8\columnwidth]{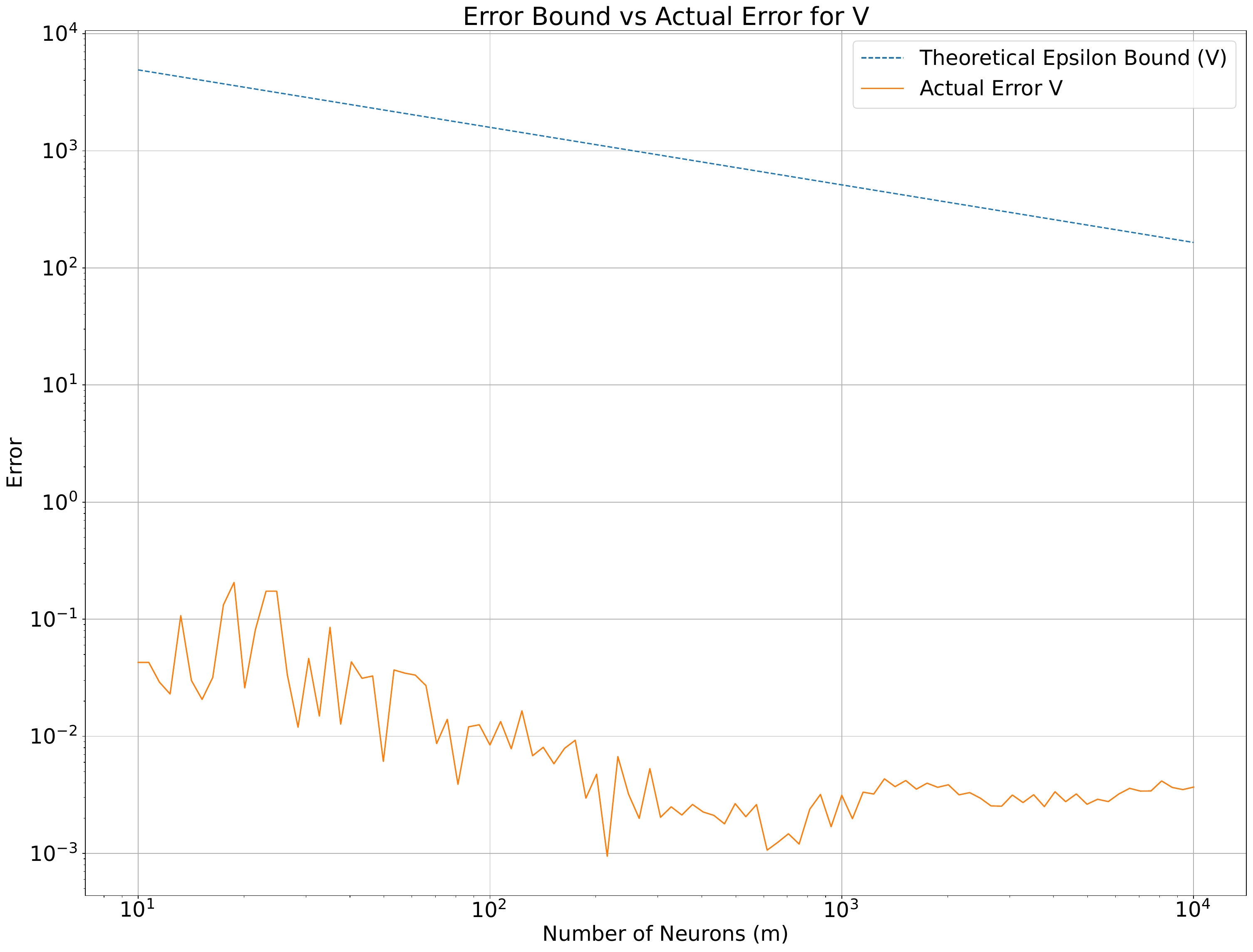}
  \caption{\label{fig:Verr} {\bf Function Approximation Error}}
\end{figure}

\begin{figure}
  \centering
  \includegraphics[width=.8\columnwidth]{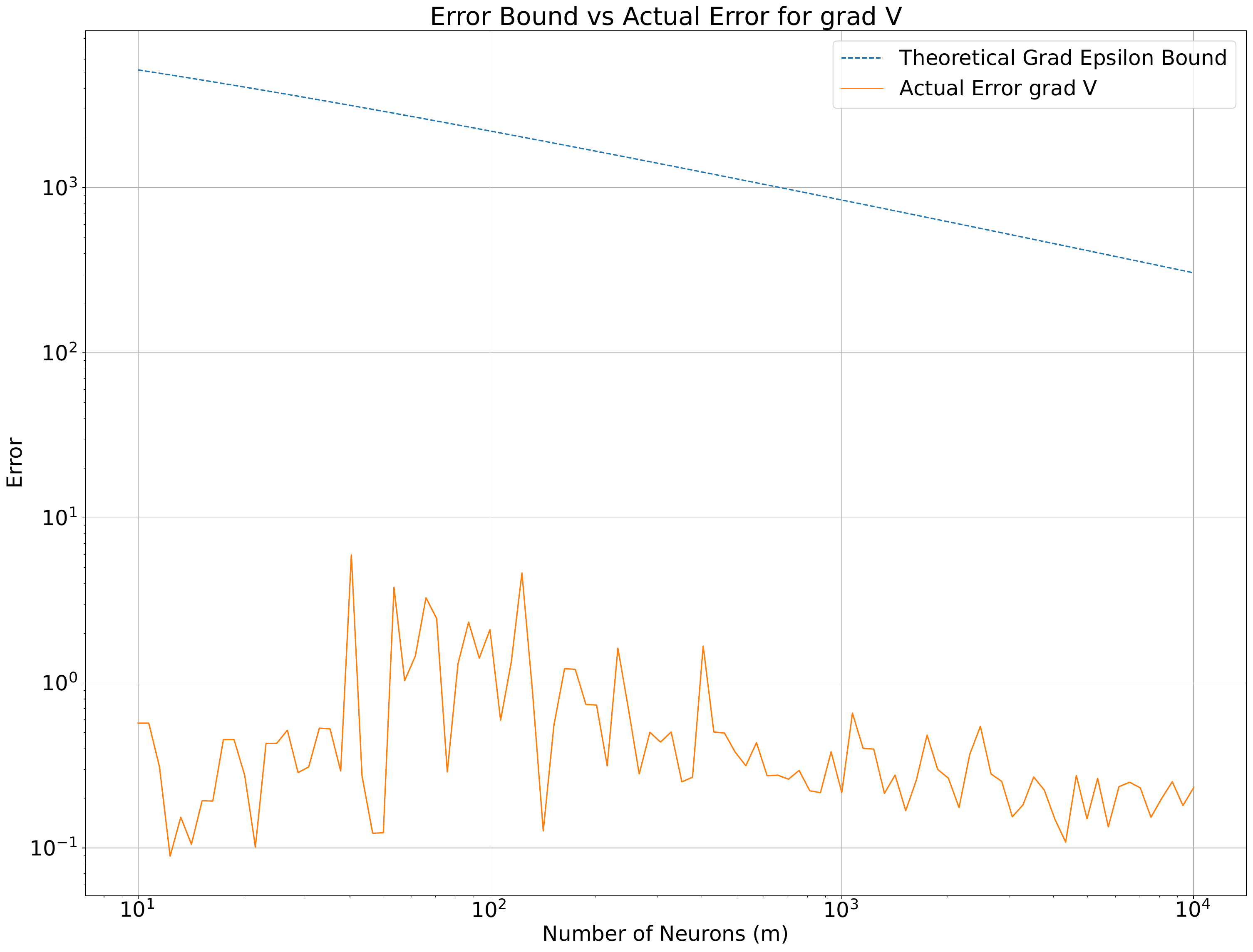}
  \caption{\label{fig:GradErr} {\bf Gradient Approximation Error}}
\end{figure}

%%% Local Variables:
%%% TeX-master: "main"
%%% End:

%\input{simulation}

\section{Conclusion}
\label{sec:conclusion}

This paper gives error bounds on both the function and gradient, when a smooth function is approximated by shallow ReLU network with randomly generated input parameters. We showed that both errors decrease like $O\left((\log(m)/m)^{1/2}\right)$, where $m$ is the number of neurons. This approximation result was motivated by policy evaluation algorithms used in dynamic programming and reinforcement learning, which require both function and gradient approximation bounds. Previously, such bounds had been assumed without proof. We showed how our results can, in principle, give rigorous guarantees that the bounds hold. However, the constant factors are currently large, which limits the practical applicability. Future work will involve improving the constant factors.

%\input{absIntroNotation}

%\input{background}

%\input{integral}
%\input{mollified}

%\input{random}

%\input{application}

%\input{conclusion}

%xf\newpage
\printbibliography

\if\ARXIV1
\newpage
\appendices
\onecolumn
\fi

\end{document}